\documentclass{article}
\usepackage[preprint]{log_2026}				

\usepackage{booktabs}						
\usepackage{multirow}						
\usepackage{graphicx}						
\usepackage{svg}
\usepackage{wrapfig}						

\usepackage{amsmath,amssymb,amsthm,mathtools}
\usepackage{enumitem}

\newtheorem{theorem}{Theorem}
\newtheorem{corollary}[theorem]{Corollary}

\newcommand{\R}{\mathbb{R}}
\newcommand{\comm}[2]{[#1,#2]}

\makeatletter
\newcommand{\ifcamerareadyonly}[1]{\if@submission\else#1\fi}
\newcommand{\ifanonymousonly}[1]{\if@submission#1\fi}
\makeatother
\usepackage[numbers,compress,sort]{natbib}	

\title[Higher-Order Positional Encodings for Graph Representation Learning]{Higher-Order Positional Encodings for Graph Representation Learning}

\author[Stam et al.]{
Caleb Stam\\
University of California, Santa Barbara\\
\email{cstam@ucsb.edu}\And
Aagrim Hoysal\\
University of California, Santa Barbara \\
\email{aagrim@ucsb.edu}\And
Sanjukta Krishnagopal\\
University of California, Santa Barbara\\
\email{sanjukta@ucsb.edu}
}

\begin{document}

\maketitle

\begin{abstract}
Many real-world systems exhibit higher-order interactions among groups of entities that cannot be captured by pairwise relationships alone. Graph Transformers and Graph Neural Networks increasingly rely on positional encodings to enrich graph representations, yet existing positional encodings are computed solely from the original graph and therefore cannot directly capture observed higher-order interactions. Topological Deep Learning addresses this limitation by lifting graphs to simplicial complexes, but typically requires performing message passing or attention on higher-order neural network representations. We introduce a representation learning paradigm that enriches graph representations with higher-order topology through positional encodings, enabling standard graph learning models to exploit lifted incidence structure without modifying the backbone. We derive a theoretical characterization of the expressivity of higher-order positional encodings, proving that node-level operators induced by higher-order lifts can mix graph Laplacian frequencies in ways that scalar graph spectral filters cannot. Guided by this theory, we instantiate higher-order positional encodings using Hodge Laplacians derived from clique complexes. Experiments with Graph Transformers on ZINC and controlled synthetic benchmarks demonstrate improvements in predictive performance, while a fixed-1-skeleton experiment shows that the pipeline can transmit higher-order information when cells are supplied independently of the graph. Together, our results establish higher-order positional encodings as a principled bridge between graph positional encodings and topological deep learning.
\end{abstract}

\section{Introduction}

Relational learning models systems of interacting entities, such as molecules, social groups, or physical and biological networks. Many such systems exhibit interactions among more than two entities at once, and their behavior is governed by this higher-order structure rather than by pairwise relationships alone \citep{battiston2020networks, battiston2021physics, benson2016higher}. For instance, species coexistence in ecological communities depends on mechanisms that cannot be reduced to pairwise competition \citep{levine2017beyond}, contagion spreads through groups rather than along single edges \citep{iacopini2019simplicial, cencetti2021temporal}, and neural population activity is organized by simultaneous firing between multiple brain regions \citep{giusti2016two}. In each case, reducing multiway interactions to pairwise edges fundamentally changes the underlying dynamics. Capturing such phenomena therefore requires representations that encode multiway interactions directly, rather than approximating them through pairwise relationships. Higher-order representations provide a language for this problem by allowing relations among more than two entities.

Graph neural networks are the default architecture for relational data, with message-passing variants such as Graph Convolutional Networks (GCNs) \citep{Kipf:2017tc}, Graph Isomorphism Networks (GINs) \citep{xu2019powerful}, and Graph Attention Networks (GATs) \citep{velickovic2018graph} aggregating features over one-hop neighborhoods. Graph Transformers \citep{dwivedi2021generalization, ying2021transformers} replace local neighborhoods with global attention, typically augmented by positional or structural encodings (PSEs) such as Laplacian positional encodings (LapPE) or random-walk positional encodings (RWPE) \citep{rampavsek2022recipe}. Both the graph itself and the PSEs computed from it encode only pairwise edges. Consequently, genuinely higher-order interactions are not captured by these representations. 

Classical Topological Deep Learning (TDL) methods perform message passing or attention directly on lifted domains \citep{papillon2023architectures, ebli2020simplicial, bunch2020simplicial, bodnar2021weisfeiler}, requiring specialized higher-order neural architectures. More recently, several works have shown \cite{bernardez2025hopse, zhou2023facilitating, pellegrin2026higher, Choi:2024tigt} that positional encodings computed on higher-order domains can improve graph learning. However, these methods are often constrained to particular types of higher-order structures, or require the neural architecture to process the higher-order structure itself. 

Our work addresses this gap by developing an operator-theoretic framework for higher-order positional encodings and a general methodology for incorporating these encodings into graph-level architectures. Rather than treating lifted positional encodings as isolated constructions, we formulate them in terms of node-level operators induced by arbitrary higher-order lifts. We ask when a lift-induced operator can be reproduced by a scalar filter $g(L)$ of the original graph Laplacian. Lifted operators that mix Laplacian eigenmodes lie outside this class.

We demonstrate that this representational difference can translate into improved downstream learning on controlled synthetic datasets and ZINC. We additionally consider complexes with the same 1-skeleton but different filled faces, directly testing the ability to use higher-order information absent from the graph skeleton.

Our contributions are:
\begin{itemize}

\item We establish a theoretical characterization of when a lift-induced operator lies outside the scalar Laplacian-filter class and quantify its distance to that class.

\item We provide an architecture-independent lift--encode--project pipeline and demonstrate its utility on controlled synthetic and molecular benchmarks.

\end{itemize}

\section{Background}

\paragraph{Graph positional encodings.}
Positional and structural encodings (PSEs) augment node representations with structural information computed from the underlying graph without modifying the graph learning architecture. Widely used graph PSEs include Laplacian positional encodings (LapPE) and random-walk positional encodings (RWPE), which derive node representations from the graph Laplacian or random-walk statistics and are commonly used with both GNNs and Graph Transformers \citep{dwivedi2021generalization, ying2021transformers, rampavsek2022recipe}. Earlier approaches include anchor-based positional encodings, which embed nodes according to their distances from sampled anchor sets \citep{you2019position}. Throughout this work, we refer to positional encodings computed directly from the original graph as \emph{flat-graph positional encodings}.

\paragraph{Topological deep learning.}
Topological Deep Learning (TDL) extends graph learning to higher-order relational domains such as simplicial complexes, hypergraphs, cell complexes, and combinatorial complexes \citep{papillon2023architectures}. Rather than propagating information only along graph edges, these methods perform message passing or attention over incidence and boundary relations, enabling models to capture higher-order interactions explicitly \citep{ebli2020simplicial, bunch2020simplicial, bodnar2021weisfeiler, hajij2022topological}. Representative examples include simplicial neural networks \citep{ebli2020simplicial}, simplicial 2-complex convolutional networks \citep{bunch2020simplicial}, and combinatorial complex neural networks \citep{hajij2022topological}. The Cellular Transformer \citep{ballester2024attending} extends self- and cross-attention directly to cell-complex incidence relations, while TIGT \citep{Choi:2024tigt} combines topological positional information with topology-specific message passing and clique-based structure within the Transformer architecture. In both cases, higher-order topology is integrated into the model's computation rather than supplied solely as an architecture-independent positional representation. This added expressivity comes at a cost: each neighborhood relation in a higher-order complex contributes a GNN-like message-passing term $O(\|\mathcal{N}\|_0 F^2 + \mathrm{nrows}(\mathcal{N})\,d\,F)$ to a layer's complexity, and a full layer sums this over every relation in the lift rather than a single adjacency structure \citep{papillon2024topotune}.

\paragraph{Higher-order positional encodings.}
A more closely related recent line of work computes higher-order positional encodings as preprocessing while retaining standard backbones.
HOPSE \citep{bernardez2025hopse} constructs positional encodings from arbitrary combinatorial complexes via Hasse graphs. \citet{zhou2023facilitating} develop random-walk and Hodge-Laplacian positional encodings on clique complexes, while \citet{pellegrin2026higher} introduce analogous random-walk, Hodge-Laplacian, and curvature-based encodings on hypergraphs. These methods study particular families of higher-order positional encodings and establish their discriminative power relative to Weisfeiler-Lehman-style baselines; our work instead asks whether the operator induced by a lift can be reproduced by \emph{any} scalar spectral filter of the original graph. More specifically, \emph{what new representational geometry does lifting make available that cannot be obtained from the original graph?}

\section{Operator Theoretic Characterization of Higher-Order Positional Encodings}
\label{sec:theory}

Existing higher-order positional encodings are primarily analyzed through graph-discrimination or Weisfeiler--Leman expressivity: they establish that particular lifted encodings distinguish graphs that weaker graph representations cannot. Our goal is different. Rather than asking which graphs a lifted encoding distinguishes, we ask a representation-theoretic question:

\textbf{What class of positional representations is obtainable from the original graph, and when does lifting fundamentally expand that class?}

The algebraic result applies to any symmetric node-level operator. We apply it to lift-induced operators, but lift provenance is not used in the equivalence itself. The commutator is therefore a test of departure from scalar filtering of a chosen Laplacian, not a characterization of higher-order origin.

Laplacian, random-walk, and heat-kernel positional encodings all process the original graph.  Our encodings instead lift the graph to a higher-order object and diffuse through its incidence relations.  We therefore ask whether lifting merely selects another operator within the family of graph spectral representations, or whether it fundamentally enlarges the representational class available from the original graph.

\paragraph{Ordinary spectral filters.}
Let $G=(V,E)$ be an undirected graph with combinatorial Laplacian $L=D-A$.  Since $L$ is symmetric, it admits an orthonormal eigendecomposition \[ L=U\Lambda U^\top, \qquad U=[u_1,\ldots,u_n], \qquad \Lambda=\operatorname{diag}(\lambda_1,\ldots,\lambda_n). \] The vectors $u_i$ are the graph analogue of Fourier frequencies.  A scalar graph spectral filter is any operator of the form \[ g(L)=U\,\operatorname{diag}\bigl(g(\lambda_1),\ldots,g(\lambda_n)\bigr)U^\top. \] This family includes polynomial filters, heat kernels, random-walk diffusion operators, and learned scalar spectral responses. We denote this family by
\begin{equation}
    \mathcal S(L) = \{g(L)\},
\end{equation}
and refer to it as the scalar $L$-filter class. Every operator in $\mathcal S(L)$ is diagonal in the Laplacian eigenbasis. Consequently, it may reweight graph frequencies independently, but it cannot couple different frequencies.

\paragraph{A lifted operator.}
Let $M\in\R^{n\times n}$ be a symmetric node-level operator induced by a higher-order lift.  As a concrete example, let $C_1$ be the unsigned vertex--edge incidence matrix and $C_2$ the unsigned edge--triangle incidence matrix.  Then \begin{equation} M_\triangle=C_1C_2C_2^\top C_1^\top \label{eq:triangle-operator} \end{equation} maps a node signal along $V\rightarrow E\rightarrow T\rightarrow E\rightarrow V$, allowing interaction through shared filled triangles.

To compare $M$ with ordinary filters, express it in the Laplacian basis: \[ \widehat M=U^\top M U, \qquad \widehat M_{ij}=u_i^\top M u_j. \] An off-diagonal entry $\widehat M_{ij}$, $i\neq j$, means that input in $u_j$ contributes to output in $u_i$. We call this \emph{frequency mixing}. Frequency mixing is impossible for scalar $L$-filters, but is not unique to higher-order constructions: ordinary graph-derived operators can also fail to commute with $L$ on irregular graphs.

\begin{theorem}[Characterization relative to scalar Laplacian filters]
\label{thm:main}
Let $L=U\Lambda U^\top$ be the Laplacian of an undirected graph, and assume that its eigenvalues are distinct. Let $u_1,\ldots,u_n$ be its orthonormal eigenvectors, and let $M\in\R^{n\times n}$ be any symmetric node-level operator. Define the commutator $[L,M]=LM-ML$.

A scalar graph spectral filter has the form \[ g(L)=U\,\operatorname{diag}\bigl(g(\lambda_1),\ldots,g(\lambda_n)\bigr)U^\top. \] 
The theorem below completely characterizes when a lifted operator belongs to the flat-graph spectral representation class.
The lifted operator $M$ falls into exactly one of two cases.

\smallskip
\noindent\textbf{Case 1: $M$ does not mix graph frequencies.} If \[ u_i^\top M u_j=0\qquad\text{for every }i\neq j, \] then $M$ only rescales the eigenvectors of $L$. In this case $M=g(L)$ for some scalar $g$, equivalently $[L,M]=0$, so a scalar filter reproduces $M$ exactly.

\smallskip
\noindent\textbf{Case 2: $M$ mixes graph frequencies.} If \[ u_i^\top M u_j\neq0\qquad\text{for at least one pair }i\neq j, \] then $M$ sends part of one Laplacian eigenmode into another, and \[ [L,M]\neq0 \qquad\text{and}\qquad M\neq g(L)\quad\text{for every scalar function }g. \] The size of the part that no scalar filter can reproduce is exactly \begin{equation} \inf_g\|M-g(L)\|_F^2 =\sum_{i\neq j}|u_i^\top M u_j|^2. \label{eq:exact-distance-main} \end{equation} This distance is zero in Case~1 and strictly positive in Case~2.
\end{theorem}

\begin{proof}[Proof sketch]

In the Laplacian basis, every scalar spectral operator is diagonal. The closest scalar filter therefore matches the diagonal entries of $U^\top M U$, leaving precisely the off-diagonal energy as the irreducible approximation error. Since

\[
[\Lambda,U^\top MU]_{ij}
=
(\lambda_i-\lambda_j)
u_i^\top M u_j,
\]

distinct eigenvalues make commutation equivalent to the absence of frequency mixing. Full details appear in Appendix~\ref{app:exact-projection}.

\end{proof}

Importantly, the simple-spectrum assumption is not essential to the separation result. With repeated eigenvalues, $[L,M]=0$ means that $M$ is block diagonal across distinct eigenspaces, while $M=g(L)$ additionally requires scalar action within every degenerate block. The exact projection replaces each block by its best scalar multiple of the identity; Appendix~\ref{app:repeated} gives the full derivation.

\paragraph{Interpretation.}

Theorem~\ref{thm:main} identifies the mechanism by which higher-order lifts expand the representational space of positional encodings.

Flat-graph spectral positional encodings manipulate graph frequencies independently: they can attenuate or amplify individual Laplacian eigenmodes but cannot create interactions between them. In contrast, lift-induced operators may couple different graph frequencies through higher-order incidence relations. Such operators therefore lie outside the entire flat-graph spectral representation class, regardless of the particular scalar filter employed.

Unlike previous expressivity analyses based on graph discrimination or Weisfeiler--Leman hierarchies, Theorem~\ref{thm:main} characterizes the representational geometry induced by lifting. Rather than asking whether two graphs become distinguishable, it asks whether the positional representation itself can be realized by \emph{any} scalar spectral operator on the original graph.


Thus Case~2 cannot be recovered by changing the polynomial degree, diffusion time, heat-kernel scale, or learned scalar response applied to the same $L$. Equation~\eqref{eq:exact-distance-main} quantifies what scalar $L$-filters leave unmatched.

\begin{figure}
    \centering
    \includegraphics[width=0.99\linewidth]{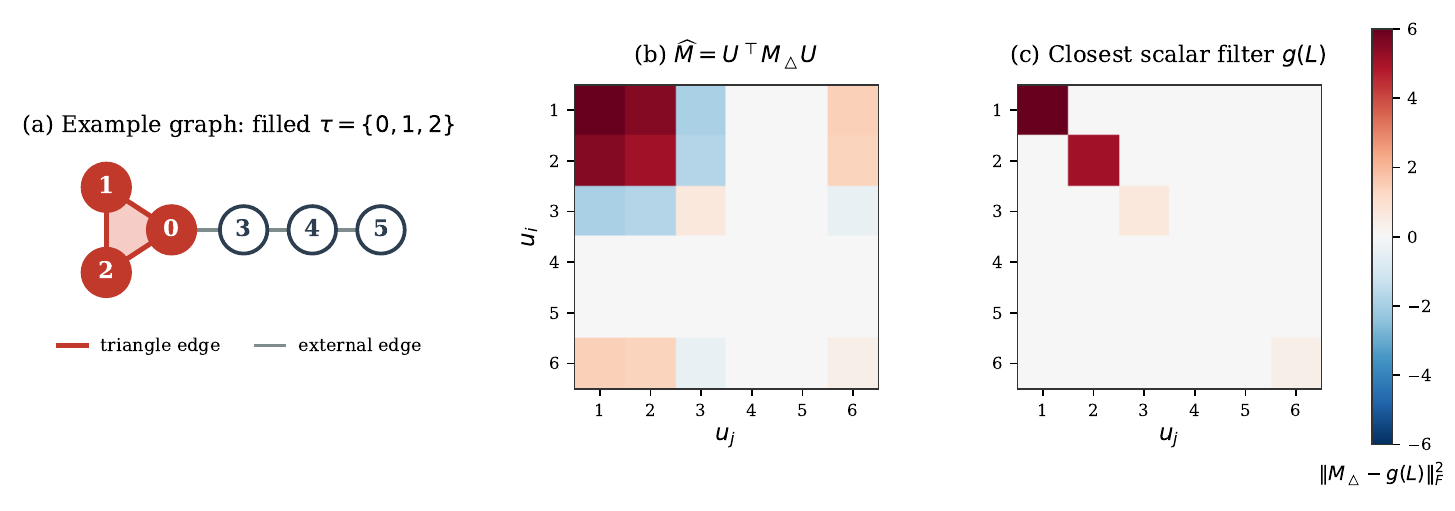}
    \caption{(a) An example graph corresponding to Corollary~\ref{cor:triangle}: filled triangle $\tau=\{0,1,2\}$, with node 0 adjacent to an external vertex. (b) $\widehat M = U^\top M_\triangle U$, the lifted operator in the Laplacian eigenbasis; off-diagonal entries are the frequency-mixing terms no scalar filter can reproduce. (c) The closest achievable scalar filter, $g(L)=\operatorname{diag}(\widehat M_{ii})$, in the same eigenbasis, sharing (b)'s color scale. The operator is diagonal; it recovers none of (b)'s off-diagonal mass.}
    \label{fig:theory_mixing}
\end{figure}

\begin{corollary}[Triangle-mediated diffusion lies outside the flat-graph spectral class]
\label{cor:triangle}

Let $G$ contain exactly one filled triangle $\tau$, with at least one triangle vertex adjacent to a vertex outside $\tau$.

Define

\[
M_\triangle
=
C_1C_2C_2^\top C_1^\top.
\]

Then
\[
LM_\triangle \neq M_\triangle L
\;\Longrightarrow\;
M_\triangle \notin \mathcal S(L).
\]

Consequently, even the simplest nontrivial higher-order lift induces frequency mixing that cannot be reproduced by any scalar spectral operator on the original graph.

\end{corollary}

\begin{proof}[Proof sketch]
For the triangle indicator $z_\tau$,

\[
M_\triangle
=
4z_\tau z_\tau^\top.
\]

The external edge implies that $z_\tau$ is not an eigenvector of $L$, yielding

\[
[L,M_\triangle]\neq0.
\]

The result then follows directly from Theorem~\ref{thm:main}. Full details appear in Appendix~\ref{app:triangle-proof}.

\end{proof}

\subsection{Design Principles for Higher-Order Positional Encodings}

Theorem~\ref{thm:main} provides more than a characterization of one particular positional encoding: it identifies a general criterion for when a higher-order lift expands the class of graph positional representations. The key question is not whether an encoding is derived from a simplicial complex, hypergraph, or another higher-order object, but whether the induced node-level operator introduces propagation geometry unavailable to scalar spectral filters on the original graph.

The theorem therefore suggests a simple design principle. Given any higher-order lift that induces a symmetric node-level operator $M$, one may ask whether
\[
[L,M]=0.
\]
If the commutator vanishes, then $M$ is equivalent to a scalar spectral filter on the original graph and introduces no new representational capacity. If $[L,M]\neq0$, then the lift necessarily enlarges the graph spectral representation class by introducing frequency mixing that no scalar graph spectral filter can reproduce.

Equation~\eqref{eq:exact-distance-main} also provides a quantitative measure of this separation: the Frobenius distance to the nearest scalar spectral filter is exactly the energy contained in the frequency-mixing terms. This offers a natural way to compare different higher-order lifts according to how far they depart from flat-graph spectral representations.

\paragraph{Generalizability.} Importantly, this characterization is independent of the particular lifting construction. It applies to any symmetric node-level operator induced by simplicial, cellular, hypergraph, combinatorial-complex, motif-based, or related higher-order lifts. The Hodge Laplacian studied in our work is therefore one representative instantiation of a broader theoretical framework rather than the unique setting in which the results apply.

\paragraph{When are higher-order PSEs useful.} 
Finally, Theorem~\ref{thm:main} is a statement about representational capacity rather than predictive accuracy. Whether the additional representational power improves learning depends on the task and the higher-order structure exposed by the lift. We therefore expect higher-order positional encodings to be most beneficial when prediction depends on genuinely higher-order interactions that cannot be recovered from pairwise structure alone. In the next section, we evaluate this hypothesis empirically using Hodge-Laplacian positional encodings on clique complexes as one representative instance of the general framework.

\section{Experiments}

The previous section established a general characterization of when higher-order lifts enlarge the class of graph positional representations. We now investigate whether this additional representational capacity translates into improved learning, instantiating the framework with the Hodge Laplacian on the clique complex.

Each of our experiments follows the same three-stage pipeline. First, a pairwise graph input is lifted to a higher-order domain; we use a clique lifting, representing the graph as a simplicial complex by treating 3-cliques as 2-simplices, so that the resulting clique complex explicitly encodes each triangle as a filled face rather than as three separate edges. Second, we compute a Hodge Laplacian positional encoding on the lifted object, described in detail below. Finally, the resulting features are passed to an unmodified backbone. Lifting and encoding are always a preprocessing step: no experiment runs message passing or attention directly on the lifted domain. The pipeline itself is agnostic to the choice of lift, encoding, and backbone; we instantiate it here with the clique complex, Hodge Laplacian encoding, and a Graph Transformer. A visualization of this architecture is shown in Figure \ref{fig:architecture_overview} and precise experimental details for each test can be found in Appendix~\ref{app:experimental-details}.

\begin{figure}
    \centering
    \includegraphics[width=0.95\linewidth]{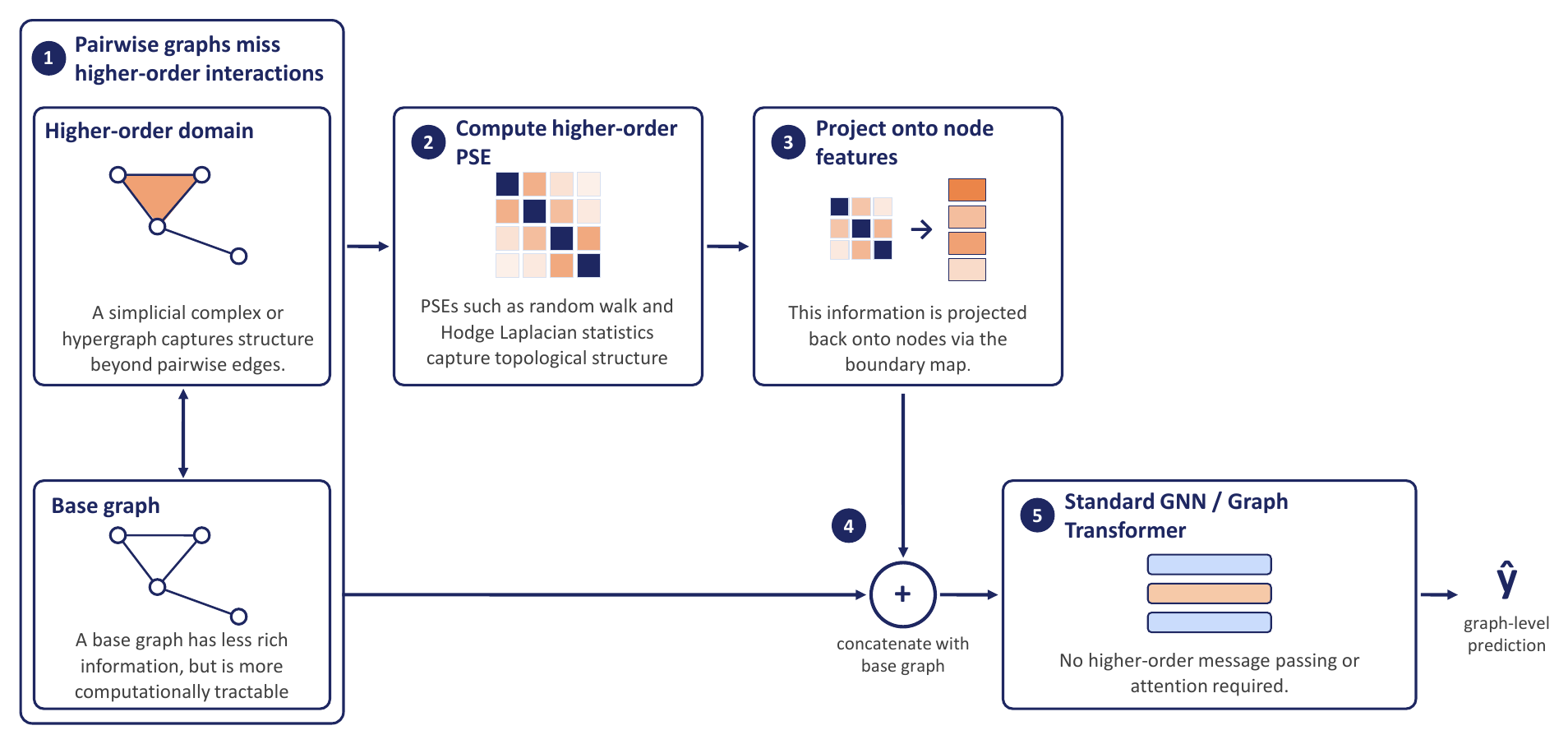}
    \caption{Overview of the proposed architecture. Positional and structural encodings are computed in the higher-order domain before being projected to the node level and passed to a standard GNN or Graph Transformer.}
    \label{fig:architecture_overview}
\end{figure}

\paragraph{Hodge Laplacian.} The combinatorial Hodge Laplacian generalizes the graph Laplacian to simplices of any dimension $k$: writing $B_k$ for the signed incidence matrix between $(k-1)$- and $k$-simplices, $L_k = B_k^\top B_k + B_{k+1}B_{k+1}^\top$ couples $k$-simplices sharing a face with those sharing a coface \citep{horak2013spectra}. The standard graph Laplacian is the $k=0$ case, while the encoding used throughout our work is the $k=1$ case, the edge Hodge Laplacian $L_1 = B_1^\top B_1 + B_2 B_2^\top$. Here, the up term $B_2B_2^\top$ is built directly from the triangle-incidence structure a clique-complex lift exposes. We choose it because its spectrum is known to carry structure invisible to the ordinary graph Laplacian such as simplicial community structure \citep{krishnagopal2021spectral}. It also shares the graph Laplacian's algebraic structure, so just as the graph Laplacian's eigenvectors underpin random-walk diffusion on nodes, $L_1$'s non-harmonic eigenvectors admit an analogous diffusion interpretation over edge space \citep{schaub2020random}.

For each selected normalized edge-Hodge eigenvector $q_j\in\mathbb{R}^{|E|}$, the implemented node feature is
\[
z_j(v)=\sum_{e\ni v}|q_j(e)|.
\]
The absolute value makes the aggregation invariant to eigenvector sign and avoids the cancellation suffered by a signed boundary projection: upper-Hodge components lie in $\operatorname{im}(B_2)$ and $B_1B_2=0$. This finite nonlinear PSE is motivated by lifted diffusion but is not algebraically identical to the linear node operator $M_\triangle$ analyzed in Theorem~\ref{thm:main}.

For a clique-complex lift, candidate 2-simplices are determined by the graph, so this construction reorganizes graph information into an edge--triangle incidence representation; it does not add information-theoretically new data. Independently supplied filled cells, by contrast, can contain information absent from the 1-skeleton.

While we choose the edge Hodge Laplacian in our experiments, the theory admits a broader class of encodings. The same incidence-and-boundary recipe extends to $L_k$ for $k>1$, coupling higher simplices through shared faces and cofaces, and more generally to signals defined over cells of any order \citep{barbarossa2020topological}. A 2-Hodge Laplacian, for instance, would couple filled triangles through shared tetrahedra, exposing structure our $L_1$ encoding cannot see. The same recipe, incidence and boundary operators composed into a symmetric cell-level operator, also extends beyond simplicial complexes to hypergraph incidence matrices and cellular boundary maps, and Theorem~\ref{thm:main}'s commutator test applies to any of them unchanged. The Hodge Laplacian studied here is therefore just one well-motivated point in a much larger design space.

\subsection{Controlled synthetic experiments}

Our synthetic graph families are designed to isolate higher-order structure while controlling pairwise statistics. They therefore provide a direct test of the practical value of the additional representational capacity characterized by Theorem~\ref{thm:main}.

\paragraph{Methods.} Our synthetic dataset consists of graphs generated to share a degree sequence while their higher-order structure varies. A base Erd\H{o}s--R\'enyi graph $G(n,p)$ is repeatedly perturbed by degree-preserving double-edge swaps: each swap removes two edges $(a,b)$ and $(c,d)$ and reconnects them as $(a,c)$ and $(b,d)$ (or the alternate pairing), subject to no multi-edges, leaving the degree sequence exactly unchanged \citep{maslov2002specificity}. A new graph is emitted at a fixed swap interval, producing a family that shares one degree sequence but spans a range of triangle densities. Because degree sequence alone cannot distinguish graphs within a family, a model that only aggregates degree-like information should perform at the mean-predictor level; any gap opened by an encoding is attributable to the higher-order structure it captures.

The regression target probes higher-order structure directly. Here, we regress against the number of 4-cliques in the graph. This quantity is not a function of the degree sequence: two graphs can share an identical degree sequence while differing arbitrarily in clique count, so an encoding that only aggregates degree-like information has in principle no basis for distinguishing labels within a family. We target 4-cliques rather than triangles specifically because triangle count is nearly built into the Hodge encoding by construction for triangle count $t$. Detecting a 4-clique, by contrast, requires recognizing coordinated overlap among triangles rather than merely tallying them, a harder demand on the encoding's structure. 

\paragraph{Results}

\begin{wraptable}{r}{0.42\textwidth}
\centering
\begin{tabular}{lc}
\toprule
Encoding & 4-clique count MAE \\
\midrule
LapPE & $6.274\pm0.001$ \\
RWPE & $6.271\pm0.008$ \\
HOPSE-M-C & $5.813\pm0.136$ \\
Hodge & $\mathbf{3.906\pm0.053}$ \\
\bottomrule
\end{tabular}
\caption{Validation MAE (mean $\pm$ std over 5 trials) for four-clique counting under degree-sequence control, $n\in[18,22]$.}
\label{tab:cliques}
\end{wraptable}

Results are shown in Table~\ref{tab:cliques}. LapPE and RWPE perform almost identically, while HOPSE-M-C, the manual, connectivity-only variant of HOPSE \citep{bernardez2025hopse}, provides a modest improvement over both. Hodge achieves substantially lower error than all three alternatives and is stable across the five trials. This supports the practical value of the edge--triangle incidence structure encoded by the Hodge PSE for recognizing coordinated overlap among triangles.

\subsection{Randomized lift experiments}
If the improvements arise from meaningful higher-order structure rather than simply a richer embedding, destroying the higher-order topology while preserving the encoding pipeline should reduce performance.

\paragraph{Methods} Here, we use the Hodge encoding to demonstrate the model's ability to leverage structural content rather than sheer encoding capacity. To do so, we build a size-matched randomized control (HodgeRandom) by replacing the true boundary matrix $B_2$ with a randomly generated matrix of the same shape, preserving the number of triangles while destroying the per-edge triangle participation structure that $B_2$ encodes. Both methods expose the total triangle count $t$ through $\operatorname{tr}(B_2B_2^\top)=3t$, so tasks correlated with $t$ alone cannot separate them. To isolate the per-edge triangle-participation structure that only the real Hodge encodes, we define a task based on the fraction of triangle-participating edges that participate in exactly one triangle: for edge $(u,v)$ with triangle count $|N(u)\cap N(v)|$, the numerator counts edges with triangle count $1$ and the denominator counts edges with triangle count $\geq 1$. Normalizing by the triangle-participating edge count removes the direct dependence on $t$, so the task requires per-edge boundary information rather than a triangle tally. The graph regime matches the 4-clique task. We compare No PSE, LapPE, RWPE, Hodge, HodgeRandom, and connectivity-only HOPSE-M-C.

\begin{wraptable}{r}{0.4\textwidth}
\centering
\scriptsize
\begin{tabular}{lc}
\toprule
Encoding & MAE \\
\midrule
No PSE & $0.115 \pm 0.001$ \\
LapPE & $0.128\pm0.001$ \\
RWPE & $0.086\pm0.025$ \\
HOPSE-M-C & $0.111\pm0.003$ \\
HodgeRandom & $0.094 \pm 0.018$ \\
Hodge & $\mathbf{0.066\pm0.001}$ \\
\bottomrule
\end{tabular}
\caption{Validation MAE (mean $\pm$ std over 5 trials) for the single-triangle-edge fraction task, $n\in[18,22]$.}
\label{tab:hodge-ablation}
\end{wraptable}

\paragraph{Results} Results are shown in Table~\ref{tab:hodge-ablation}. Hodge achieves the lowest error across the five trials, outperforming HodgeRandom, HOPSE-M-C, LapPE, and RWPE. This confirms that the real Hodge PSE encodes structurally meaningful per-edge triangle participation that the randomized control cannot recover, consistent with Theorem~\ref{thm:main}'s distinction between frequency-mixing content and any operator that merely matches a scalar summary statistic. The residual improvement of HodgeRandom over no positional encoding likely reflects indirect correlations with triangle count that survive normalization imperfectly, rather than any structural content.

\subsection{Molecular dataset experiments}
Finally, we investigate whether the representational advantages identified theoretically and observed on controlled synthetic graphs translate to a practical molecular prediction task.

\paragraph{Methods} The ZINC dataset experiment \citep{irwin2012zinc} is a graph-level task that predicts the penalized logP (the molecule's water-octanol partition coefficient minus its synthetic accessibility score and the number of cycles of more than six atoms) from molecular graphs, using the standard train/validation/test split of Benchmarking Graph Neural Networks \citep{dwivedi2022benchmarking}. Penalized logP is used to train molecular generation models, as in \citet{pmlr-v80-jin18a, pmlr-v70-kusner17a}.
Each node represents a heavy atom, with node features one-hot encoded by atom type. We evaluate both the original setting without edge features and a setting with standard bond features. The 12{,}000 molecular graphs have a mean node count of 23.2 and edge count of 49.8. We use MAE of the regressed versus actual penalized logP as our performance metric.

Since cycle count is part of the calculation for penalized logP, we test the ability of lifted positional encodings to distinguish larger graph structures than flat-graph positional encodings.

\begin{wraptable}{r}{0.52\textwidth}
\centering
\scriptsize
\begin{tabular}{lcc}
\toprule
Encoding & No bond features & Bond features \\
\midrule
No PSE & $0.694\pm0.002$ & $0.556\pm0.041$ \\
LapPE & $0.683\pm0.010$ & $0.637\pm0.016$ \\
RWPE & $0.686\pm0.003$ & $0.585\pm0.014$ \\
Full Hodge & $0.534\pm0.014$ & $0.473\pm0.024$ \\
Lower Hodge & $0.523\pm0.018$ & $\mathbf{0.405\pm0.030}$ \\
\bottomrule
\end{tabular}
\caption{ZINC test MAE (mean $\pm$ standard deviation) with and without standard bond features.}
\label{tab:zinc}
\end{wraptable}

\paragraph{Results} Results are shown in Table~\ref{tab:zinc}. Results shown are averaged over five trials. We compare the full edge Hodge Laplacian with its lower component $B_1^\top B_1$ to distinguish the effect of edge-space spectral information from that of the triangle-dependent upper component $B_2B_2^\top$.

Full and Lower Hodge outperform No PSE, node LapPE, and RWPE both with and without incorporating bond features. Lower Hodge performs best, identifying edge-space spectral lifting rather than the triangle upper term as the main source of the ZINC gain. This is consistent with the triangle statistics of the dataset: 94.0/92.5/94.3\% of train/validation/test molecules contain no triangles, only 0.3/0.7/0.1\% contain at least two, and the corresponding mean counts are 0.06/0.08/0.06. Thus most examples provide no $B_2$ signal, and the ZINC results primarily support the value of an edge-space inductive bias.

\section{Conclusion}

We introduced an operator-theoretic framework and a lift--encode--project construction for supplying higher-order positional information to standard graph-learning backbones. The commutator criterion characterizes when a lift-induced operator mixes graph Laplacian eigenspaces beyond the reach of scalar spectral filters, while the projection formula quantifies this separation. Controlled experiments show the value of lifted incidence structure, and the fixed-skeleton setting demonstrates that the pipeline can transmit higher-order information supplied independently of the graph. On ZINC, the Lower-Hodge result further shows that edge-space spectral structure can remain useful when triangle incidence is sparse. Together, these findings establish lifted positional encodings as a flexible bridge between higher-order representations and conventional GNNs or Graph Transformers.

\paragraph{Limitations.} Our empirical study is a proof of concept centered on one lift and encoding: the clique complex and Hodge Laplacian. Performance depends on whether the selected lift exposes structure relevant to the task; a poorly matched lift may add cost without improving prediction. In addition, constructing dense complexes and computing their spectra can make preprocessing expensive, even though the resulting encodings can be cached and used with an unchanged backbone.

\paragraph{Future work.} A systematic empirical evaluation across datasets, higher-order domains, positional encodings, and backbone architectures is an important next step. This includes studying hypergraphs, cellular and combinatorial complexes, and data in which higher-order relations are observed directly rather than inferred through a lift. It would also be useful to compare preprocessing-based encodings with architectures that operate directly on the lifted domain, and to develop sparse or partial spectral methods that improve scalability on larger complexes.

\paragraph{Code availability.}
Code is available at \url{https://github.com/ahoysal/lifted-pses}.

\bibliographystyle{unsrtnat}
\bibliography{reference}

\clearpage
\appendix
\section{Experimental details}
\label{app:experimental-details}

\subsection{Compute resources}

Experiments were run on a compute cluster using SLURM. Each job was allocated a single node with one NVIDIA Tesla V100 GPU (32\,GB VRAM, PCIe), 8 CPU cores, and 32\,GB RAM. Cluster nodes have dual-socket Intel Xeon CPUs (40 logical cores per node, 192\,GB RAM total).

For 5{,}000 graphs, preprocessing requires 452.2s for Hodge, 6.0s for LapPE, 14.4s for RWPE, and 12.9s for HOPSE. Once the encodings are cached, matched Graph Transformers require 1.6--1.7s per epoch and approximately 57.6\,MB peak GPU memory for all encodings; cache storage is $O(nd_{\rm PE})$. Thus the downstream per-epoch training time is effectively equivalent across encodings: Hodge's additional overhead is incurred only once during preprocessing, while yielding the substantial performance improvements reported in Tables~\ref{tab:cliques}--\ref{tab:zinc}. CPU Hodge preprocessing grows from 3.1s at $n=10$ to 1135.8s at $n=70$ as triangle density rises. Sparse or partial eigensolvers are a promising direction for reducing this preprocessing cost on dense complexes.

\ifcamerareadyonly{Use was made of computational facilities purchased with funds from the National Science Foundation (CNS-1725797) and administered by the Center for Scientific Computing (CSC). The CSC is supported by the California NanoSystems Institute and the Materials Research Science and Engineering Center (MRSEC; NSF DMR 2308708) at UC Santa Barbara.}

\subsection{ZINC}

Dataset: 10,000 train, 1,000 validation, and 1,000 test molecular graphs of up to 38 heavy atoms, from the standard split of \citet{dwivedi2022benchmarking}.

Model: Graph Transformer, $d_{\text{model}}=128$, 2 attention heads, 2 layers, dropout 0.2, linear input projection, global sum-pooling readout. Optimizer: Adam, learning rate $10^{-3}$, L1 (MAE) loss, batch size 32, 300 epochs with best checkpoint restored. All experiments use 5 independent trials per encoding.

Encoding dimensions: LapPE, 20 eigenvectors; RWPE, 20 anchors, 3-step; Hodge, 20 eigenvectors.

\subsection{Synthetic experiments}
Model: Graph Transformer, $d_{\text{model}}=64$, 4 attention heads, 3 layers, dropout 0.1, linear input projection, global sum-pooling readout. Optimizer: Adam, learning rate $10^{-3}$, L1 (MAE) loss, batch size 256, up to 100 epochs with early stopping (patience 20 epochs, best checkpoint restored), 5 independent trials per encoding-task pair.

Encoding dimensions: LapPE, 8 eigenvectors; RWPE, 16 anchors, 3-step; Hodge, 8 eigenvectors, with a size-matched HodgeRandom control.

Dataset generation: a base Erd\H{o}s--R\'enyi graph $G(n,p)$ is perturbed by degree-preserving double-edge swaps, 2 per step, with a new graph emitted every 5 swaps. Task-specific parameters:

On the additional four-cycle-counting diagnostic, five-trial validation MAE is $81.089800\pm0.591537$ for HOPSE-M-C, $97.780400\pm0.265618$ for LapPE, $97.510200\pm0.178438$ for RWPE, and $\mathbf{52.207437\pm7.587771}$ for Hodge.

\begin{table}[h]
\centering
\begin{tabular}{lccc}
\toprule
Task & $n$ & $p$ & Train / eval graphs \\
\midrule
Four-clique counting & $[18,22]$ & $[0.3,0.4]$ & 5{,}000 / 1{,}000 \\
Hodge ablation (single-triangle-edge fraction) & $[18,22]$ & $[0.3,0.4]$ & 5{,}000 / 1{,}000 \\
\bottomrule
\end{tabular}
\caption{Synthetic dataset generation parameters by task.}
\label{tab:synth-params}
\end{table}

\subsection{Fixed skeleton with varying filled faces}

To isolate genuinely higher-order input, we construct examples with the same 1-skeleton but different selections of filled 2-simplices. Graph positional encodings and all statistics computed solely from the skeleton are therefore identical across examples; only $B_2$ changes. The targets measure the organization of the filled faces through their number of connected components and the fraction belonging to the largest component. Across both targets, graph-only methods remain near mean-predictor performance, whereas the Hodge encoding substantially reduces prediction error (see Table~\ref{tab:fixed-skeleton}). Because the pairwise graph is fixed, this performance gap shows in direct and easily interpretable terms that the lift--encode--project pipeline can transmit higher-order information supplied independently of the graph skeleton.

\begin{table}[t]
\centering
\scriptsize
\begin{tabular}{lcc}
\toprule
Encoding & Num.\ components & Largest-comp.\ fraction \\
\midrule
No PSE & $1.037\pm0.0003$ & $0.131\pm0.0002$ \\
RWPE & $1.036\pm0.0003$ & $0.130\pm0.0002$ \\
LapPE & $1.036\pm0.0003$ & $0.131\pm0.0002$ \\
\midrule
Hodge & $\mathbf{0.307\pm0.013}$ & $\mathbf{0.057\pm0.010}$ \\
\bottomrule
\end{tabular}
\caption{Fixed-skeleton validation MAE (mean $\pm$ std over 5 trials) on both targets. All examples share one 1-skeleton; only $B_2$ varies. The first block uses graph-only information and the second uses the supplied filled faces.}
\label{tab:fixed-skeleton}
\end{table}

\section{Complete proofs and extensions}
\label{app:proofs}

\subsection{Exact projection onto scalar spectral filters}
\label{app:exact-projection}

Let $L=U\Lambda U^\top$ have simple spectrum and define $\widehat M=U^\top MU$.  The Frobenius norm satisfies $\|Q_1AQ_2\|_F=\|A\|_F$ for orthogonal $Q_1,Q_2$.  This follows from \[ \|QA\|_F^2 =\operatorname{tr}((QA)^\top QA) =\operatorname{tr}(A^\top Q^\top QA) =\operatorname{tr}(A^\top A), \] and the right-multiplication statement is analogous.  Consequently, \begin{align*} \|M-g(L)\|_F^2 &=\left\|U^\top\bigl(M-g(L)\bigr)U\right\|_F^2\\ &=\left\|\widehat M- \operatorname{diag}\bigl(g(\lambda_1),\ldots,g(\lambda_n)\bigr)\right\|_F^2\\ &=\sum_i|\widehat M_{ii}-g(\lambda_i)|^2 +\sum_{i\neq j}|\widehat M_{ij}|^2. \end{align*} Because $\lambda_1,\ldots,\lambda_n$ are distinct, arbitrary prescribed values at these finitely many points can be realized by a scalar function $g$.  More strongly, Lagrange interpolation realizes them with a polynomial of degree at most $n-1$.  Taking $g(\lambda_i)=\widehat M_{ii}$ eliminates the first sum, which proves \[ \inf_g\|M-g(L)\|_F^2 =\sum_{i\neq j}|\widehat M_{ij}|^2 =\sum_{i\neq j}|u_i^\top Mu_j|^2. \] The minimum is zero exactly when $\widehat M$ is diagonal.

It remains to relate diagonality to commutation.  Direct calculation gives \[ U^\top\comm{L}{M}U =\comm{\Lambda}{\widehat M}, \qquad \comm{\Lambda}{\widehat M}_{ij} =(\lambda_i-\lambda_j)\widehat M_{ij}. \] Since the spectrum is simple, $\lambda_i-\lambda_j\neq0$ whenever $i\neq j$.  Thus $\comm{L}{M}=0$ exactly when $\widehat M$ is diagonal.  In that case, interpolation supplies a $g$ for which $M=g(L)$, completing the proof of Theorem~\ref{thm:main}.

\subsection{A quantitative separation certificate}
\label{app:quantitative}

The preceding identity also yields \[ \|\comm{L}{M}\|_F^2 =\sum_{i\neq j}(\lambda_i-\lambda_j)^2 |\widehat M_{ij}|^2. \] Since $|\lambda_i-\lambda_j|\leq \lambda_{\max}(L)-\lambda_{\min}(L)$, \begin{align*} \|\comm{L}{M}\|_F^2 &\leq \bigl(\lambda_{\max}-\lambda_{\min}\bigr)^2 \sum_{i\neq j}|\widehat M_{ij}|^2\\ &=\bigl(\lambda_{\max}-\lambda_{\min}\bigr)^2 \inf_g\|M-g(L)\|_F^2. \end{align*} Therefore \begin{equation} \inf_g\|M-g(L)\|_F \geq \frac{\|\comm{L}{M}\|_F} {\lambda_{\max}(L)-\lambda_{\min}(L)}. \label{eq:comm-bound} \end{equation} For a nontrivial graph Laplacian, $\lambda_{\min}(L)=0$ and $\lambda_{\max}(L)>0$.  Hence a nonzero commutator not only proves inequality with every scalar filter; it certifies a strictly positive Frobenius-distance gap from the whole family.

\subsection{Full calculation for the triangle example}
\label{app:triangle-proof}

Assume that the lift contains one filled triangle $\tau$.  Then $C_2$ has a single column $b_\tau$, with one nonzero entry for each of the three edges of $\tau$.  Therefore \[ M_\triangle=C_1b_\tau b_\tau^\top C_1^\top =(C_1b_\tau)(C_1b_\tau)^\top. \] Every triangle vertex is incident to exactly two triangle edges, while every other vertex is incident to none of them.  Thus $C_1b_\tau=2z_\tau$ and \[ M_\triangle=4z_\tau z_\tau^\top. \]

We use the following elementary rank-one fact.  For symmetric $L$ and nonzero $z$, \begin{equation} \comm{L}{zz^\top}=0 \quad\Longleftrightarrow\quad z\text{ is an eigenvector of }L. \label{eq:rank-one-fact} \end{equation} For the forward direction, apply $Lzz^\top=zz^\top L$ to $z$: \[ Lz\,\|z\|_2^2=z(z^\top Lz). \] Hence $Lz$ is proportional to $z$.  Conversely, if $Lz=\lambda z$, symmetry of $L$ gives $Lzz^\top=\lambda zz^\top=zz^\top L$.

Now choose an outside vertex $v\notin\tau$ adjacent to at least one triangle vertex.  Since $(z_\tau)_v=0$, \[ (Lz_\tau)_v =d_v(z_\tau)_v-\sum_uA_{vu}(z_\tau)_u =-\sum_{u\in\tau}A_{vu} =-|N(v)\cap\tau|<0. \] Thus $Lz_\tau$ cannot be proportional to $z_\tau$.  Equation \eqref{eq:rank-one-fact} implies $\comm{L}{z_\tau z_\tau^\top}\neq0$, and hence $\comm{L}{M_\triangle}\neq0$.  The general separation implication in Theorem~\ref{thm:main} gives $M_\triangle\neq g(L)$ for every scalar $g$; Equation~\eqref{eq:comm-bound} additionally gives a positive quantitative gap.

\subsection{Repeated eigenvalues}
\label{app:repeated}

Simple spectrum keeps the main theorem easy to state, but noncommutation remains a valid separation certificate without it.  Let $\{\mu_r\}_{r=1}^q$ be the distinct eigenvalues of $L$, and let $P_r$ be the orthogonal projector onto the eigenspace of $\mu_r$.  Every scalar filter has the form \[ g(L)=\sum_{r=1}^q g(\mu_r)P_r. \] Within a repeated eigenspace, $g(L)$ must act as one scalar times the identity.  For a symmetric $M$, the best scalar for that block is \[ \alpha_r= \frac{\operatorname{tr}(P_rMP_r)}{\operatorname{rank}(P_r)}. \] The approximation residual therefore contains two kinds of terms: the cross-eigenspace blocks $P_rMP_s$ for $r\neq s$, and the non-scalar part $P_rMP_r-\alpha_rP_r$ within each repeated eigenspace.  Consequently, $\comm{L}{M}\neq0$ is still sufficient to prove $M\neq g(L)$.  The converse is no longer automatic: $M$ may commute with $L$ while acting non-scalar within a repeated eigenspace.

\subsection{Generality and scope}
\label{app:scope}

The argument requires only a symmetric graph reference operator $L$ and a symmetric lifted node operator $M$.  It therefore covers a wide range of operators induced by simplicial complexes, cell complexes, hypergraphs, motif incidence structures, and related lifts; the triangle operator in Corollary~\ref{cor:triangle} is merely the simplest illustration.  The universal direction---$\comm{L}{M}\neq0$ implies separation from every $g(L)$---also holds when $L$ has repeated eigenvalues.  Normalized reversible walks can be handled through their symmetric similarity transforms or the corresponding weighted inner product.  Fully directed, nonsymmetric operators require the analogous nonsymmetric formulation.

The theorem establishes an operator-level representational separation.  A finite positional encoding inherits this advantage to the extent that its chosen statistics retain the distinguishing operator information; truncation, pooling, and the downstream architecture determine how much is used in a particular implementation.  The controlled experiments complement the theorem by measuring this effect after the complete encoding and learning pipeline.

Finally, deterministic lifts need not introduce external data to create a new representation.  Their benefit is precisely that they reorganize the graph into an operator with cross-frequency interactions outside the scalar graph spectral-filter family.

\end{document}